\documentclass[article]{llncs}
\usepackage[T1]{fontenc}
\usepackage{graphicx}
\usepackage{amsmath,amssymb,amsthm}
\usepackage{xcolor}
\usepackage{algorithm}
\usepackage{algpseudocode}

\newcommand{\T}{\mathcal{T}}  
\newcommand{\A}{\mathcal{A}}  
\newcommand{\I}{\mathcal{I}}  

\theoremstyle{definition}
\newtheorem{exmp}{Example}
\begin{document}
\title{Incremental, inconsistency-resilient reasoning over Description Logic Abox streams}
%
%
 \author{ Cas Proost\orcidID{0009-0007-8419-7363} 
 \and Pieter Bonte\orcidID{0000-0002-8931-8343}
}
\authorrunning{C. Proost et al.}
%
\institute{Departement of Computer Science, KU Leuven Kulak, Belgium
\email{pieter.bonte@kuleuven.be}}
\maketitle              
\begin{abstract}
More and more, data is being produced in a streaming fashion. This has led to increased interest into how actionable insights can be extracted in real time from data streams through Stream Reasoning. Reasoning over data streams raises multiple challenges, notably the high velocity of data, the real time requirement of the reasoning, and the noisy and volatile nature of streams. This paper proposes novel semantics for incremental reasoning over streams of Description Logic ABoxes, in order to tackle these challenges. To address the first two challenges, our semantics for reasoning over sliding windows on streams allow for incrementally computing the materialization of the window based on the materialization of the previous window. Furthermore, to deal with the volatile nature of streams, we present novel semantics for inconsistency repair on such windows, based on preferred repair semantics. We then detail our proposed semi-naive algorithms for incremental materialization maintenance in the case of OWL2 RL, both in the presence of inconsistencies and without.

\end{abstract}
\section{Introduction}



The present-day digital landscape is defined by the incessant and ever-expanding generation of data streams. Social media, stock market, online shopping platforms, etc, all are generating information on a scale that was unheard of 20 years ago \cite{della2009s}. Because of the shear volume and velocity of these data streams, any party attempting to extract actionable insights from a stream, finds itself confronted with a monumental task. It is exactly this problem that led to the inception of the field of stream reasoning. It combines the field of stream processing with the field of knowledge representation and reasoning. Stream processing has long been focused on extracting information from streams, but on a relatively superficial level.  While superficial insights that can be explicitly read from the data can be interesting, most data streams contain treasure troves of implicit, deducible information. Deducing this implicit information is done through the process of reasoning, which relies on a given set of logical rules, representing domain knowledge. The more detailed or complex the insights that need to be extracted, the more expressive the logical rules, and the harder the reasoning problem. This tradeoff between expressiveness and efficiency is particularly important when reasoning over streams, because of the real-time requirement caused by the dynamic nature of the data: the insights need to be deduced and acted upon before the underlying data stream has changed (too much). Managing the expressiveness - efficiency tradeoff while also considering the growing demand for more complex insights is a major challenge in the field of stream reasoning \cite{bonte2024grounding}. Stream reasoning research therefore often focuses on optimizing reasoning algorithms, to allow for more expressive reasoning while still fulfilling the real-time requirement.\\
One of the most widespread tools in stream reasoning research, is the so-called window operator: Streams are unbounded, as they can theoretically keep going forever. It is therefore impossible to wait until the whole stream has been produced, and then process it in its entirety. The idea of a window operator is to chop a data stream into finite, processable chunks. Instead of waiting eternally for the whole stream to arrive, windowing allows reasoners to process the stream chunk by chunk.\\
While there are many different window operators, the most basic is the time based window. A time-based window operator in its most fundamental form states a starting time and an end time, and returns all facts in the stream between those times. In real-time applications, the end time of a window is commonly (close to) the current time, and the starting time is some set duration, e.g. five minutes, into the past. Such a window allows the reasoner to consider only the most recent five minutes of data for the purpose of deducing logical insights. Depending on the application, a window can move in various increments, for example the above five minute window could have its start and end time moved in increments of one or two minutes, in which case the next window overlaps the current window, referred to as a sliding window. A five minute window that only moves every five minutes is a special case where no overlap is present between consecutive windows, and is called a tumbling window.\\
Given a window, a reasoner may then perform a task called materialization, which consists of deducing all facts entailed by the data in the window~\cite{motik2019maintenance}. Such a materialization can then be queried to find all explicit and implicit answers to a query. Unfortunately, because a stream is continuously changing, when the window slides, the old window materialization is no longer accurate. New facts have entered the window and old facts have left, thus a new materialization needs to be computed. Most stream reasoners will do this from scratch, but some make use of incremental window maintenance techniques \cite{barbieri2010incremental,dell2014incremental,laser}. These techniques make use of the old materialization to compute the new: when a five minute window slides by one minute, there is a four minute overlap between two subsequent windows. Instead of computing the new window materialization from scratch, the idea is to work incrementally: remove the data that has left the window, and any facts deduced from that data, and then add the new data that has entered, and any facts entailed by that new data. This idea allows for more efficient reasoning, especially when subsequent windows have large overlaps.\\
Another major challenge of stream reasoning is the volatile and noisy nature of data streams. Real-life streams are often generated by vast amounts of sensors, which occasionally pass erroneous readings into the stream, leading to contradictions often referred to as inconsistencies. Standard symbolic reasoning techniques are not capable of dealing with inconsistencies, which means any erroneous information passed into a reasoning engine can cause the whole system to shutdown.  Moreover, inconsistencies can also arise from the way reasoning engines process data streams:\\
For simplicity purposes, many SR approaches such as C-SPARQL\cite{barbieri2009csparql}, C-QELS\cite{lephuoc2011native}, $\text{SPARQL}_{\text{stream}}$\cite{corcho2012enabling}, IMaRS\cite{dell2014incremental}, use windows to select the desired data, and then reason over the data within the window as if it is a static knowledge base. This idea is known as reasoning under graph-entailment \cite{dell2017stream}, and it is a powerful assumption that allows for more scalability and efficiency. It does, however, have its drawbacks, as illustrated by following example: suppose we consider a data-stream consisting of sensor readings of a nuclear reactor. Different sensors in the reactor could produce their data in varying intervals, suppose for example that a temperature sensor indicates whether the reactor temperature is too low, normal, or too high, every 30 seconds. Now consider a window that selects the most recent minute out of that stream. That window would contain two temperature readings. If these readings are identical, no problem will arise when the data within the window is viewed as a static knowledge base. However, if the temperature readings differ, a static view of the window would 2 contradicting facts, e.g. the reactor temperature is too low and the reactor temperature is too high. A reasoner without the capability for inconsistency repair will state that this window is inconsistent, and any automated systems that rely on the reasoner's results to regulate the reactor temperature will not be able to act, with potentially catastrophic consequences. \\
This example shows that inconsistencies in data-streams can originate from many different sources, not exclusively from erroneous sensor readings. Handling inconsistencies is an important and as of yet unsolved, challenge in the stream reasoning field as is brought up in \cite{bonte2024grounding} and \cite{dell2017stream}.\\
\\
This paper aims to address the above mentioned two of the major challenges of the stream reasoning field. For this purpose, we concern ourselves with a family of expressive logics called Description Logics (DLs). DLs have been extensively studied, but largely in the context of static data, where there is no real-time requirement. Because of their expressive nature, they are well suited for applications that require extracting complex insights from data, which is a fundamental objective of stream reasoning. For example, they form the basis of the Web Ontology Language (OWL2)\cite{Hitzler2012OWL2W} that is widely used in the semantic web and stream reasoning communities.\\
Naturally, due to the expressiveness-efficiency tradeoff, reasoning under (especially the most expressive) DLs is difficult \cite{hustadt2005complexity,calvanese2013complexity}. We therefore seek to revise DL semantics to support incremental reasoning over sliding windows, allowing us to retain the expressiveness of DLs while simultaneously satisfying real-time performance requirements, addressing the first challenge.\\
Moreover, repairing inconsistencies has been extensively studied in DL literature \cite{lembo2010inconsistency,rosati2011complexity}. By leveraging the ideas of preferred repair semantics \cite{bienvenu2014querying} and adapting them to our novel semantics, we address the second challenge: the volatile and noisy nature of streams.
\\
This paper is structured as follows: We start by outlining the issues with related research in both the incremental reasoning area and in inconsistency repair. We then introduce preliminary basics on DL and streams. Then we detail our incremental semantics for reasoning over windows, followed by the inconsistency repair semantics. In the final section we sketch semi-naive algorithms for reasoning under these semantics in the description logic RL, which corresponds to the OWL2 RL profile.

\section{Related work}
Incremental reasoning or materialization maintenance does not have much presence in the description logics literature, as it is more extensively studied in the database and knowledge representation communities \cite{salem2000rollin,griffin1995incremental,ceri1991deriving},where reasoning takes a different form. However, it has seen some attention in reasoning research under the datalog formalism, which is similar to description logics (even equivalent to some DLs). In \cite{motik2019maintenance}, the authors give a comprehensive overview and comparison of different algorithms for maintaining the materialization of a datalog knowledge base. They consider three distinct approaches: the counting algorithm, the delete/rederive algorithm (DRed) and the Forward Backward Forward algorithm (FBF). \\
The counting algorithm keeps track of how many times a fact in the knowledge base has been derived/proven, and when facts are deleted, reduces this counter for all facts that derive from the deleted facts. When a fact's counter hits zero, it also disappears from the knowledge base. This algorithm seems to perform among the best of the three, but it has the added overhead of keeping these counters in memories, which unfortunately is not benchmarked in the paper. \\
The DRed algorithm starts instead by deleting all facts that derive from the to be deleted facts, a step known as overdeletion, and then rederiving any facts that were overdeleted. The FBF algorithm is a parametrized algorithm that acts similar to DRed by starting with a deletion step, but it seeks to limit overdeletion by performing a backward chain on facts that get added to the deletion pool. This backward chain looks for a proof that will hold after deletion, and if such a proof is found, the fact is removed from the deletion pool. A parameter decides the depth of the backward chain, meaning that DRed is the extremal case of DRed with backward chaining depth equal to zero. \\
Note that these algorithms need to perform reasoning in order to delete facts, as they need to figure out which facts in the materialization depend on the to be deleted facts, so that they can also be removed/have their count reduced. This is not ideal in a streaming scenario where latency is crucial. This is why IMaRS\cite{dell2014incremental} and LASER\cite{laser} developed algorithms for incremental materialization maintenance on streams that do not need to perform reasoning for the deleting of facts.\\
For this, they make use of the temporal nature of the stream. Because materialization maintenance is often performed on windows in a streaming scenario, the timestamps of facts indicate which ABox assertions will leave the window when it slides. By taking into account the timestamps of the ABox assertions used to derive a fact, it is therefore possible to determine a "horizon time" when the fact will leave the ABox. This time is determined when the fact is derived and updated whenever a different proof for the fact is found. Therefore, when the window slides and facts need to be deleted from the materialization, there is no need to reason which derived facts need to be deleted, it is sufficient to look for the facts with expired horizon times and then delete those. \\
These two papers are the only one to implement such algorithms for streams however, and they do not cover a wide variety of formalisms: IMaRS works on RDFS+ and LASER for LARS. DL literature has seen no such algorithm nor even semantics specifically adapted to streams. This despite the fact that the DL-based OWL2 ontology language is widely used within the semantic web and the IoT, which are two of the main application domains for stream reasoning.  Furthermore, neither IMaRS, LASER or the incremental view maintenance algorithms (Counting, Dred, FBF) are able to perform reasoning when inconsistencies arise, even though inconsistencies are common occurrences in volatile data streams.\\
\\
In contrast, the DL field has seen widespread interest in the problem of inconsistency repair. Different semantics seek to enable different reasoning tasks over inconsistent knowledge bases by “repairing” the inconsistencies. Most work considers only inconsistent ABoxes, as an inconsistent TBox is generally seen as a modeling mistake \cite{bienvenu2014querying,lembo2010inconsistency}. In the case of stream reasoning, an inconsistent TBox can and should be fixed before deployment of the system, while inconsistencies in the ABox (stream) have to be repaired on the go.\\
Current DL inconsistency repair semantics stem from the idea of ABox Repair (AR) \cite{lembo2010inconsistency} semantics, which considers maximal consistent subsets of the ABox, known as repairs. Facts are then considered to be entailed by an inconsistent ABox if it is entailed by all repairs of that ABox. Some variations on this idea include the Intersection AR (IAR) semantics, and brave semantics, which consider the same maximal consistent subsets but define a fact to be entailed if it holds in the intersection of all repairs, respectively if it holds in any repair. The problem with these semantics is that they require the consideration of all these repairs, and hence the computation of all these repairs in order to perform materialization \cite{Benferhat2015preferred}. Furthermore, the computation of these repairs is currently done by "guessing" subsets of the ABox and checking if they are repairs. Since an ABox with $n$ assertions has $2^n$ possible subsets, this process can be extremely inefficient in some cases. \\
An interesting modification of AR semantics is the idea of preferred repair semantics\cite{bienvenu2014querying}, which is a generalization of AR or IAR semantics where repairs are also maximal consistent subsets of the ABox, but with maximality defined with respect to some general preorder on subsets rather than the $\subseteq$ preorder. The most well-known preferred repair are the priority semantics, which assumes different priority levels within the ABox, and the weight repair semantics, which assigns weights to all assertions in the ABox. In \cite{bienvenu2014querying}, the authors investigate the complexity of query answering under these preferred repair semantics. They find that these semantics lead to larger sets of answers, with decreased performance. We note however that their implementation seems to behave in a way similar to AR or IAR algorithms, but with extra steps, which explains the performance decrease.\\
While AR, IAR and brave semantics have been studied for streaming DL in \cite{bourgaux2019ontology}, they only did so for $\text{DL-Lite}\mathcal{R}$ and $\mathcal{EL}\bot$. Moreover, most work on inconsistency repair focuses on query answering rather than materialization. Preferred repair semantics has, to the best of our knowledge, not been studied over streams, nor for materialization, nor in the context of OWL2 RL. 

\section{Preliminaries}
\subsection{Description logics and OWL2 RL}

DLs are a family of knowledge representation languages commonly used in reasoning applications for their properties of decidability and expressiveness\cite{Baader_Horrocks_Lutz_Sattler_2017}. In a standard DL system, there are two distinct structures of information: an ABox and a TBox. The ABox consists of factual information, ABox assertions, which usually are one of two things: concept assertions or role assertions. Concept assertions take the form $C(x)$, meaning 'the individual $x$ belongs to the concept $C$. DL concepts correspond to unary predicates in other formal knowledge representation languages such as e.g. datalog. In contrast, role assertions take the form $R(x,y)$, meaning "individual $x$ is $R$-related to individual $y$, role therefore being the DL equivalent of binary predicates. \\
While the ABox contains factual, explicit information, the TBox contains rules used to reason over these facts. A DL rule consists of a head and a body, which, depending on the syntax of the particular description logic, are sentences constructed from concepts, roles, and logical operators such as $\sqcap, \sqcup, \exists, \forall, \dots$ TBox rules allow for reasoning over data found in an ABox: if the body/subsumee of a rule holds for a certain individual (or pair of individuals), then the head/subsumer must hold too. \\
DL semantics are based on the idea of interpretations, which represent a knowledge base by mapping it to some concrete set: an interpretation $\I$ consists of a domain $\Delta_\I$ and a mapping $\cdot^\I$. This mapping maps every individual in the knowledge base to an element of the domain, every concept to a subset of the domain, and every role to a subset of $\Delta_\I\times\Delta_\I$. In this paper, we will uphold the unique name assumption (UNA), which means that the domain of an interpretation is exactly the set of individuals in the knowledge base, and an interpretation maps each individual to itself. Note that in general however, this need not be true when studying DLs in general.\\
An interpretation $\I$ is said to satisfy an ABox concept assertion $C(x)$ if $x^\I \in C^\I$, and a role assertion $R(x,y)$ if $(x^\I,y^\I) \in R^\I$. Naturally, if an interpretation satisfies all assertions in an ABox $\A$ it is said to satisfy $\A$. Similarly, an interpretation is said to satisfy a rule $B \sqsubset H$ if $B^\I \subset H^\I$, and it satisfies a TBox $\mathcal{T}$ if it satisfies every rule in $\mathcal{T}$. Note that $B^\I$ and $H^\I$ are somewhat ambiguous here, but will be made concrete when talking about the RL description logic this paper concerns itself with. \\
An interpretation $\I$ for a knowledge base $(\A, \mathcal{T})$ is a model for that knowledge base if it satisfies both $\A$ and $\mathcal{T}$. In DL, an assertion or query is entailed by a knowledge base if and only if it holds true in every model of that knowledge base. This means the problem of reasoning in DL comes down to computing which facts hold in every model. This is not in general easy, but for some DLs, the existence of a minimal model, i.e. a model contained in every other model, makes this problem a lot more approachable.\\
One such DL is the underlying DL of the OWL2 RL profile, of which we consider a simplified version as defined in \cite{Kontchakov2014}. In order to construct a concept or role, it uses an alphabet of concept names $C_1, ..., C_n$, role names $P_1, ..., P_n$ and the logical operators $\sqcap, \exists$ and the inverse role operator $^-$. Using this syntax, a rule is of one of the following forms:
\[ B \sqsubseteq A, B \sqsubseteq \bot, R_1 \sqsubseteq R_2\]
where $A$ is a concept name, $R_1,R_2$ are role names or their inverse, $\bot$ is the bottom concept, i.e. the concept mapped to the empty set by every model, and $B$ is a concept constructed with the operators mentioned above. The most noteworthy property of this syntax is that it does not allow for existential quantifiers in the head of a rule, which results in the useful property of the existence of a minimal or canonical model for any RL knowledge base. \\
This model is constructed starting from an interpretation that represents only the fact present in the ABox, and then iteratively adding facts to satisfy rules in the TBox. 
\begin{definition}
    For a RL knowledge base $(\A,\mathcal{T})$, the standard interpretation $\I_\A$ is defined by:
    \begin{enumerate}
                \item $A^{\I_\A} = \left\{a^{\I_\A} \vert A(a)\in \A\right\}$
                \item $P^{\I_\A} = \left\{(a^{\I_\A},b^{\I_\A}) \vert P(a,b)\in \A\right\}$
    \end{enumerate}
\end{definition} 

The standard interpretation is then used to construct a series of interpretations through a kind of forward chaining process: set $\I_0 = \I_{\A}$. To obtain $\I_{j+1}$ from $\I_j$, apply the following rules:
\begin{enumerate}
    \item for any $B\sqsubseteq A\in \mathcal{T}$, if $a\in B^{\I_j}$ and $a \notin A^{\I_j}$, add $a$ to $A^{\I_{j+1}}$
    \item for any $P_1 \sqsubseteq P_2 \in \mathcal{T}$, if $(a, b) \in P_1^{\I_j}$ and $(a,b)\notin P_2^{\I_j}$, add $(a,b) $ to $P_2^{\I_j}$
    \item for any $B\sqsubseteq \bot\in \mathcal{T}$, if $a\in B^{\I_j}$ then the process terminates as $(\A,\mathcal{T})$ is inconsistent
\end{enumerate}
Note that $B^\I$ here is the obvious choice depending on what shape $B$ takes: if $B := A$ is a concept name, $B^\I = A^\I$, if $B:= B_1\sqcap B_2$ then $B^\I = B_1^\I\cap B_2^\I$ and if $B := \exists R.B_1$, then $B^\I = \left\{a \space \vert\space  \exists b \in B_1^\I s.t. (a,b) \in R^\I\right\}$\\

As is shown in \cite{Kontchakov2014}, this procedure terminates and either results in showing the knowledge base is inconsistent or in a model for the knowledge base called the canonical model, which has the desired property of being minimal.\\
We now recall some notions of semantics for inconsistency repair. In particular, we discuss preferred repair. As with most ABox repair semantics, preferred repair semantics defines a repair to be a consistent subset of the ABox that is maximal. The relation/order that determines this maximality is what sets preferred repair semantics apart.
\begin{definition}[\cite{bienvenu2014querying}]\label{def:pref_order}
    Given an ABox $\A$ and a partitioning $\left\{\mathcal{P}_1, ..., \mathcal{P}_n\right\}$ of $\A$, called preference levels, we define a partial order $\subseteq_\mathcal{P}$ on the subsets of $\A$,$2^\A$, as follows:\\
    \[\forall \A', \A'' \subseteq \A: \A' \subseteq_\mathcal{P} \A'' \text{ iff:}\]
    \begin{enumerate}
        \item $\A'\cap\mathcal{P}_i = \A''\cap\mathcal{P}_i$ for all $1\le i \le n$ or
        \item $\exists i, 1 \le i \le n$ such that $\A'\cap\mathcal{P}_i \subsetneq \A''\cap\mathcal{P}_i$ and $\A'\cap\mathcal{P}_j = \A''\cap\mathcal{P}_j$ for all $i\le j \le n$
    \end{enumerate}

\end{definition}

A preferred repair of an ABox $\A$ with preference levels $\mathcal{P}_1, ..., \mathcal{P}_n$ is then a consistent subset $\A'\subset \A$ such that there is no other consistent subset $\A"\subset \A$ with $\A'\subseteq_\mathcal{P} \A"$. Moreover, a fact is said to be entailed under preferred AR semantics if it is entailed by every preferred repair of $\A$, and under preferred IAR if it is entailed by the intersection of all preferred repairs of $\A$.\\
Let us demonstrate this with an example:

 \begin{exmp}
     Let $\A = \mathcal{P}_1\cup \mathcal{P}_2$ with $\mathcal{P}_1 = \{A(a), B(a)\}$ and $\mathcal{P}_2 = \{C(a)\}$. let $\mathcal{T} = \{A\sqcap C \sqsubset \bot, B\sqcap C \sqsubset \bot\}$ stating that the concept $C$ is disjoint with both $A$ and $B$. The consistent subsets of $\A$ are then exactly equal to $\mathcal{P}_2$, $\mathcal{P}_1$ and the two singleton subsets of $\mathcal{P}_1$. Clearly, $\mathcal{P}_1\subseteq_\mathcal{P}\mathcal{P}_2$, as $\mathcal{P}_1\cap\mathcal{P}_2 \subsetneq \mathcal{P}_2\cap \mathcal{P}_2$. Therefore $\mathcal{P}_2$ is the only preferred repair of ($\mathcal{T},\A$), which means the only fact entailed under both preferred AR and preferred IAR by this knowledge base is $C(a)$.
 \end{exmp}
 As this example illustrates, a preferred repair of an ABox $\A = \mathcal{P}_1, ..., \mathcal{P}_n$ resolves inconsistencies by removing as little high priority assertions as possible: if an inconsistency requires removing either multiple low priority assertions or a single high priority assertion, preferred repair semantics removes the low priority assertions in favor of the high priority assertions.\\

\subsection{Streams in DL}
In general stream reasoning, a stream is a collection of time annotated facts. These facts can take different shape depending on the formalism used, eg. an RDF stream will consist of quadruplets $(s,p,o,t)$ which modify the traditional (subject predicate, object) triplets to include a timestamp. In DL however, facts are represented by ABox assertions, and therefore a stream in DL will consist of time annotated ABox assertions, which we specify below. Streams are usually considered to be infinite, which necessitates a way to bound sections of the stream. A time based window operator on a stream is defined by some interval $[t_1, t_0]$ and selects from the stream all facts with a timestamp that lies inside this interval. Reasoning over streams is often preformed under a graph entailment regime, which says a window of the stream entails a query or fact if and only if that query is entailed by the knowledge base obtained from the window. In general, graph entailment is one of three entailment regimes commonly considered in stream reasoning research. The other regimes are window-level entailment and stream-level entailment. In window level entailment, a window is still used to select a relevant part of the stream, and assertions outside the scope of the window are therefore not considered for entailment purposes. Window-level entailment differs from graph-level entailment in the fact that time stamps are not forgotten within the window however, meaning that assertions within the window are deemed to be true only for the timestamps at which they occur, as opposed to for the whole window. This entailment regime is often used in a context of more expressive, temporal reasoning over streams while still aiming for more practical reasoning, by utilizing a window to limit the scope of the reasoning. This makes window-level entailment a middle ground between the efficient, less expressive graph-level entailment and the most expressive, but impractical stream level entailment: in stream level entailment, all information present in the stream is considered for the reasoning process, including assertions from any part of the stream and their time stamps. While this makes for the most expressive reasoning, allowing for temporal relations of unlimited distance to be drawn, it is often an impractical choice for real-world use cases because of memory and computational demands \cite{dell2017stream}. \\
In this work, we will model a stream as a set $\left\{\A_t\right\}_{t\in \mathbb{T}}$ of DL ABoxes, where the set $\mathbb{T}$ is called the timeline of the stream. It is important to stress that the stream does not contain any TBox data, and when we perform reasoning over such a stream it will be in relation to a fixed TBox. The ABoxes in a stream all share from the same sets of individual names $a_1, \dots$, concept names $C_1, \dots$ and role names $P_1, \cdots$. They share these names in the sense that when a name occurs in one ABox, and later occurs in another ABox, then these names represent the same individual/concept/role, but not every ABox needs to mention every name, nor does every name need to be mentioned in multiple ABoxes. Note that our model is equivalent to a set of time annotated ABox assertions, as each ABox $\A_t$ is just the container for all assertions in the stream with timestamp $t$.\\
For the purposes of this paper, we make the following assumptions on every DL stream:
\begin{enumerate}
    \item The timeline $\mathbb{T}$ is a countably infinite subset of $\mathbb{R}$ for which there is a $\delta_0$ such that $\forall t_1,t_2 \in \mathbb{T}: \vert t_1-t_2\vert \ge \delta$.
    \item Every ABox $\A_t$ in the stream contains a finite amount of assertions
\end{enumerate}
A window defined by an interval $I = [t_1, t_2]$ is defined as the set $W_I = \left\{\A_t\right\}_{t\in I\cap \mathbb{T}}$. Associated to a window $W_I$ is its window ABox $\A_{W_I} = \bigcup\limits_{\A_t\in W_I} \A_t$. Under our assumptions, it is clear that every window and every window ABox is finite. Moreover, it is clear that under a graph-entailment regime, a query/fact is entailed by the window if and only if it is entailed by the window ABox under standard DL semantics.\\

\section{Incremental semantics}
In this section we, define an incremental version of graph-entailment semantics for window-based reasoning over description logic streams. While standard DL semantics could technically be used to reason under graph-entailment, they disregard the connection between two overlapping windows and therefore do not suffice for incremental reasoning purposes.\\
Note that these semantics are aimed to be applicable to general DL. When we discuss reasoning algorithms however, we will limit ourselves to the OWL2 RL description logic.\\

Given a stream $\left\{ \A_t\right\}_{t\in\mathbb{T}}$ as defined above and a DL TBox $\mathcal{T}$, we say an interval $I= [t_1,t_2]\cap\mathbb{T}$ defines a window $W_I = \left\{ \A_t\right\}_{t\in I}$. A window interpretation stream for $W_I$ is a set of interpretations $\left\{ \I_t\right\}_{t\in I}$ where each $\I_t$ is an interpretation (not a model) for $\A_t$. A window interpretation stream induces an interpretation on the window ABox $\A_{W_I} = \bigcup\limits_{\A_t\in W_I}\A_t$, which we can show after defining a bit of notation:
\begin{definition}
    Given two interpretations $\I_1, \I_2$ with domains $\Delta_1, \Delta_2$ for two ABoxes $\A_1,\A_2$, we define the direct sum of these interpretations as the interpretation $\I_1\oplus \I_2$ for $\A_1\cup \A_2$ with domain $\Delta_1\cup\Delta_2$ with the interpretation map given by:
    \[C^{\I_1\oplus\I_2} := C^{\I_1}\cup C^{\I_2}\]
    \[R^{\I_1\oplus\I_2} := R^{\I_1}\cup R^{\I_2}\]
    
\end{definition}
This is well-defined since we assume the unique name assumption, such that the name of the same individual is identical over different interpretations. Moreover, this definition can be used to inductively define $\bigoplus\limits_{1\le i\le n}\I_i$ for some series of interpretations $\I_1, ..., \I_n$. This allows us to define the interpretation on a window ABox induced by a window interpretation stream as follows:\\
\begin{definition}[\textbf{Window interpretation}]
    For a window $W_I$ and a window interpretation stream $\left\{ \I_t\right\}_{t\in I}$, the interpretation $\I_{W_I} = \bigoplus\limits_{t\in I}\I_t$ is the window interpretation induced by this window interpretation stream.
\end{definition}
This window interpretation is therefore an interpretation fro the window Abox $\A_W$, which contains exactly the information contained in the window interpretation stream. Notably, if the window interpretation stream were to consist of models for each ABox in the window, the associated window interpretation need not be a model for $\A_W$ which can be illustrated with the following trivial example:
\begin{exmp}
    Consider a TBox $\T =\{A\sqcap B \sqsubseteq C\}$ and a window containing two ABoxes: $\A_1 = \{A(a)\}, \A_2 =\{B(a)\}$. Suppose we consider a window interpretation stream which consists of the interpretations $\I_1$ with domain $\{a\}$ and map given by $A^{\I_1} = \{a\}, B^{\I_1} = C^{\I_1} =\emptyset$ and $\I_2$ with domain $\{a\}$ and map given by $B^{\I_2} = \{a\}, A^{\I_2} = C^{\I_2} =\emptyset$. Naturally, these interpretations are models for their respective ABoxes. However, it is clear that the induced window interpretation $\I_{W} = \I_1\oplus\I_2$ is not a model for $\A_W = \A_1\cup \A_2$, as it maps $C$ to $\emptyset$.
\end{exmp}

We can now define exactly what it means to reason over windows under graph entailment in our semantics by introducing the concept of window models.\\
 \begin{definition}[Window model]
     A window interpretation stream $\left\{ \I_t\right\}_{t\in I}$ is a window model stream for $W_I$ if each $\I_t$ is a model for $\A_t$ and the associated window interpretation $\I_{W_I}$ is a model for $\A_{W_I}$.
 \end{definition}

A fact is then said to be graph-entailed in a window $W$ if it holds in all window models for $W$, i.e. all window interpretations associated to a window model stream. The following proposition says that this is equivalent to the fact being entailed by $\A_{W}$.
\begin{proposition}\label{prop:surjection}
    The map that maps a window model stream for $W$ to its associated window interpretation $\I_{W}$ is a surjection from the window model streams of $W$ to the models of $\A_{W}$.
\end{proposition}
\begin{proof}
    It is clear that every window model stream maps to a model for $\A_{W}$, as this is how window model streams are defined. Now suppose $\I_W$ is a model for $\A_{W}$. Consider the window interpretation stream $\left\{\I_t\right\}_{t\in I}$ where each $\I_t = \I_W$. It is clear that its window interpretation is a model for $\A_{W}$. Moreover, since $\I_W$ satisfies every assertion in $\A_{W}$, it therefore satisfies every assertion in $\A_t$ for all $t\in I$, which means every $\I_t$ must be a model for $\A_t$. Hence this window interpretation stream must be a window model stream. This shows the surjection.\qed
\end{proof}

It might seem like these semantics are unnecessarily convoluted, as one could argue that considering just the models of $\A_{W_I}$ allow perfectly well for graph-entailment reasoning over windows. However, the reason these semantics are formulated in this way can be glimpsed by considering the incremental reasoning problem, i.e., what happens when the window slides?\\
Suppose only the models of $\A_{W_I}$ were considered, and materialization is done by computing the minimal/canonical model of this ABox. When the window changes, it is not possible to compute the new materialization from the old one: none of the temporal structure is preserved, it is not possible to determine which facts need to be dropped from the window.\\
In contrast, if the materialization is in the form of a window model as described above, it is easy to drop the facts that leave the window by dropping the models corresponding to the timestamps that left the window. This will be discussed in more detail in the incremental reasoning section.\\
Preferred repair semantics are not often discussed in literature, possibly because the additional structure of the preference levels does not naturally arise in the standard case of static ABoxes. In the case of a stream of ABoxes however, we do have a natural way of setting the preference levels: the timestamps of the ABoxes. If we consider how we might repair inconsistencies that arise when performing graph entailment reasoning over a window on a stream, it is clear that the partition of the window into its ABoxes can serve as the partition into preference levels. A more recent ABox is preferred over an older one, which means we can set the preference level of an ABox to be equal to its timestamp. A repair of the window will then prefer deleting any number of older facts rather than deleting a single new fact in order to resolve an inconsistency. The philosophy behind this can be pithily worded as:
\begin{center}
    \textit{New is always better.}
\end{center}
Concretely, when an inconsistency arises, the oldest fact(s) leading to that inconsistency will be deleted in favor of keeping the newer facts. This is because the newer facts are seen as more reliable, and they can be interpreted as disproving the older fact(s). 
\begin{exmp}
    Suppose we have a stream representing the operational status of a car's pedals, with concepts $ClutchPressed, GasPedalPressed, BreaksPressed$. And suppose we have the following stream of ABoxes (with the subscripts of the ABoxes representing their timestamp in seconds), that might represent a car $x$ stopping for a traffic light: 
    \[\A_0 = \{GasPedalPressed(x)\}\]
    \[\A_1 = \{GasPedalPressed(x)\}\]
    \[\A_2 = \{\}\]
    \[\A_3 = \{BreaksPressed(x)\}\]
    \[\A_4 = \{BreaksPressed(x),ClutchPressed(x)\}\] 
    If we were to have a window of size 2 seconds that slides every second, the first window would contain $\A_0, \A_1$ and $\A_2$. This window only contains the assertion $GasPedalPressed(x)$, albeit at two different times. When the window slides however $\A_0$ is dropped from the window and $\A_3$ is added. Now the window contains two distinct assertions:
    \[GasPedalPressed(x), BreaksPressed(x)\]
    Of course, anyone who has driven a car knows this is not a situation that should occur, and we might model that with a TBox containing a rule $GasPedalPressed\cap BreaksPressed \subseteq \bot$. With this rule, the window becomes inconsistent under graph entailment, and hence we would not be able to reason over what the cars current status actually is. For this reason, we say that the most recent assertion is probably the right one, and we will drop $GasPedalPressed(x)$ from the window in order to resolve the inconsistency.
\end{exmp}
The idea of interpreting an arising inconsistency as old information being disproven by new information is a central here. In particular, the window could slide again, with a new inconsistency arising that disproves facts previously used to disprove older facts. In this case, we should again delete the oldest facts that caused the new inconsistency. However, the facts that were originally removed should not be added again: disproving a disproof is not a proof, confusing as it sounds.\\
This approach also resonates better with the idea of incremental reasoning: when the window slides, the new ABoxes can be added one by one, repairing after each added ABox, rather than having to consider repairs for adding all ABoxes at once. This will again be discussed in more detail in Section~\ref{inconsistency-semantics}. For now, we adapt the semantics of preferred repairs to our case of incremental window maintenance.\\
Suppose again that $W = \left\{\A_t\right\}_{t\in I}$ is a window over a DL stream, defined by a interval $I$. Similarly to definition~\ref{def:pref_order}, we define an order on the subsets of the window:
\begin{definition}[Window order]
    The window order $\subseteqq_W$ is a partial order on the subsets of $\A_W$ defined by:\\
    \[\forall U, V \subseteq \A_W: U \subseteq_W V \text{ iff:}\]
    \begin{enumerate}
        \item $U\cap\A_t = V\cap\A_t$ for all $t\in I$ or
        \item $\exists t_0 \in I \le n$ such that $U\cap\A_{t_0} \subsetneq V\cap\A_{t_0}$ and $U\cap\A_t = V\cap\A_t$ for all $t>t_0$
    \end{enumerate}
\end{definition}

This definition establishes the preference levels for any window to be determined by the ABoxes contained in that window. Using this order, we can now define window repairs, which are based on the idea of preferred repairs. Because of the subtleties discussed above however, there is a notable difference with the original concept.
\begin{definition}[Window repair]
    Given a window $W =\left\{\A_t\right\}_{t\in I}$, a window repair is a set $W' = \left\{\A'_t\right\}_{t\in I}$ where each $\A'_t\subseteq \A_t$ that satisfies the following conditions:
    \begin{enumerate}
        \item $\A_{W'}$ is consistent
        \item For every $t\in I$, $\bigcup\limits_{s\in I,s\le t}\A'_s$ is the intersection of all maximal (with respect to $\subseteq_W$) consistent subsets of $\A_t\cup\bigcup\limits_{s\in I,s< t}\A'_s$
    \end{enumerate}
\end{definition}

In stead of taking preferred repairs of $\A_W$ with respect to the window order $\subseteq_W$, a window repair is built incrementally by repairing the window after adding each ABox, and then adding the next ABox to the repair. This keeps in line with the philosophy we discussed earlier where facts that get removed from the window during the repairing process, do not get added back just because their "counter evidence" is removed during a repair later on. Moreover, this incremental structure allows repairing algorithms to leverage the repairs made in the previous window whenever the window slides, which is extremely useful.

 \section{Incremental reasoning over windows in OWL2 RL}
 
This section describes an algorithm for reasoning incrementally under the modified DL semantics defined in the previous section. More specifically, we focus on the DL underlying the OWL2 RL profile. In order to build an incremental algorithm that also resolves inconsistencies, we need three components:
\begin{enumerate}
    \item  We need to be able to add the ABoxes that enter the window when it slides.
    \item We need to be able to remove the facts that leave the window when it slides.
    \item  We need to be able to resolve inconsistencies that arise from adding ABoxes when the window slides.
\end{enumerate} 

We start by describing the procedure for adding a single ABox to a materialized window, assuming no inconsistencies are present. This procedure will also allow us to construct a window materialization from scratch by starting with an empty materialization and adding ABoxes one by one.

\subsection{Adding an ABox to a window}

The materialization of a single RL ABox with respect to a RL TBox can be represented by its canonical model: as every model of the ABox must contain its canonical model, it holds that every fact represented by the model must be entailed by the knowledge base, and every fact not in the canonical model is not entailed by it. To represent a graph-entailed materialization for a window, we therefore want to construct the canonical model of the graph-entailment window ABox $\A_W$. To do this, we will construct a "minimal" window model whose associated graph-entailment window model is equal to the canonical model for $\A_W$. Such a minimal window model is not necessarily unique, recall Proposition~\ref{prop:surjection}. However, we will choose our construction in a way that will allow for the required incremental reasoning procedures to be quite natural. We will therefore call the model obtained through this construction the canonical window model.

When a window $W$ consists of a single ABox $W =\left\{\A\right\}$, there is only one such possibility for a minimal window model: the canonical model for $\A$. We now inductively construct a window model for a window that contains more than one ABox:\\
\\
Let $W =\left\{\A_{t_1}, ..., \A_{t_n}\right\}$ be a window consisting of $n>1$ ABoxes, where $t_i<t_j$ if $i<j$. Let $\left\{\I_1, ..., \I_{n-1}\right\}$ be the window model obtained through this construction of the window $W \setminus\A_{t_n}$, and let $\I_n$ be the canonical model of $\A_{t_n}$. We construct a window model of $W$ as follows:\\
\begin{enumerate}
    \item Add $\I_n$ to the window model.
    \item For every inclusion $B\sqsubseteq A$ (or $P\sqsubseteq R$) in $\mathcal{T}$: if there is some minimal set of models $\I_{u_1}, ..., \I_{u_m}$ in $\left\{\I_1, ..., \I_{n-1}, \I_n\right\}$ such that $x\in B^{\bigoplus\limits_i\I_{u_i}}$ for some $x$ (or $P(x,y)$ for some $x,y$), then add $x$ to $A^{\I_{\min\limits_{1\le i\le m}(u_i)}}$ if it was not in that set yet (or add $(x,y)$ to the set representing $R$ for that same interpretation).
    \item Stop this procedure when there is no rule for which this procedure will change the window model (i.e. when the head of every rule has been added to every appropriate model).
\end{enumerate}
     Note that by minimal set of models, we mean that leaving any of the models out results in $x$ no longer being in the image of $B$ under the mapping given by the direct sum of these interpretations.
\begin{definition}[Canonical window model]\label{def:CanWinModel}
    The window model obtained through the above construction is the canonical window model for $W$.
\end{definition}

The idea behind this definition is that when a fact is derived, we add it to the model of the oldest assertion used in the derivation. This ensures that if the oldest assertion used in the derivation leaves the window after a window slides, we do not need to calculate all facts derived from that assertion in order to drop them from the window. Instead, we can simply drop the model corresponding to that oldest assertion, which will take care of all facts derived from that assertion. This means that by smartly choosing how we handle the first requirement (adding ABoxes to a window), we have also solved requirement 2 (removing ABoxes from a window), which is summarised in the following proposition:
\begin{proposition}
    Let $W =\left\{\A_{t_1}, ..., \A_{t_n}\right\}$ be a window model and $\left\{\I_1, ..., \I_{n}\right\}$ be its canonical window model. Then $\left\{\I_2, ..., \I_{n}\right\}$ is the canonical window model for $W' = W\setminus \{\A_{t_1}\}$.
\end{proposition}
\begin{proof}
    Let $\{\I'_2, ..., \I'_n\}$ be the canonical window model for $W'$. We prove that $\I'_i = \I_i$. It is clear that for each concept/role $C^{\I_i} \supseteq C^{\I'_i}$: if the second step in the construction of the canonical window model was applicable to any rule during the construction of the canonical model for $W'$, the step must have also been applicable for that same rule in the construction of the model for $W$. On the flipside, it is clear that if that step was applicable for a rule in the construction for W but not for W', then it must have involved $\I_1$. Therefore $\I_i\subseteq \I'_i$ (with this slight abuse of notation we mean that the mapping of any concept or role under $\I_i$ is a subset of the mapping of that same concept or role under $\I'_i$) which concludes the proof.\qed
\end{proof}

Let us now illustrate these concepts with an example:

\begin{exmp}
    Let $S$ be a stream with timeline $\mathbb{N}$ where $\A_i =\emptyset$ for $i>4$ and let the first four ABoxes in the stream be as follows:
    \begin{align*}
        \A_1&= \{A(a),B(a)\}\\
        \A_2&= \{C(a)\}\\
        \A_3&= \{A(a)\}\\
        \A_4&= \{B(a)\}
    \end{align*}
    Consider the TBox $\mathcal{T} = \{A\sqcap C \sqsubseteq D, B\sqcap D \sqsubseteq E\}$. Suppose we want to reason over the window $W_1$ defined by $[1,2]$. To construct the canonical window model of $W_1$, we start by calculating the canonical models $\I_1,\I_2$ for $\A_1,\A_2$ respectively. It is clear that both are the standard interpretation for their respective ABox. To get the canonical window model for $W_1$ then, we start with with $\left\{\I_1\right\}$ and add $\I_2$ to get $\left\{\I_1,\I_2\right\}$ as per Definition~\ref{def:CanWinModel}. Then, we see that the rule $A\sqcap C \sqsubseteq D$ can be applied because $a\in A^{^\I_1}$ and $a \in C^{\I_2}$. Therefore we add $a$ to $D^{\I_1}$ ($\I_1$ contains the oldest facts used in this derivation). We now check again if any rules can be applied and see that $B\sqcap D \sqsubseteq E$ can be applied, wholly within $\I_1$. Therefore we add $a$ to $E^{\I_1}$, after which no more rules can be applied, and the construction terminates.\\
    The model for $W_1$ could then be used to construct the model for the window $W_2$ defined by $[1,3]$: we start with the canonical model for $W_1$ and add to it $\I_3$, the canonical model of $\A_3$, which is again equal to its standard interpretation. We check again which rules can be applied, and see that $A\sqcap C \sqsubseteq D$ can be applied to $a\in C^{\I_2}$ together with $a\in A^{\I_3}$. Note that even though this rule has been applied to this individual, it has not been applied to these two assertions together, which is why we do add $a$ to $D^{\I_2}$. \\
    Now let the window slide, meaning we are now interested in $W_3$ defined by $[2,4]$. To get the canonical model of $W_3$, we first need the canonical model of the window $W_4$ over the interval $[2,3]$. Conveniently, because of the construction of the canonical window model, we can obtain this model by dropping $\I_1$ from the canonical model for $W_2$: if we were to calculate the canonical window model for $W_4$ from scratch, we would start with the canonical model of $\A_2$, then add the canonical model of $\A_3$, and then apply any rules that can be applied, i.e. $A\sqcap C \sqsubseteq D$. Clearly however, this has been covered in the construction of the model for $W_2$, and moreover, by dropping $\I_1$ from that model, we no longer have any facts derived from $\A_1$. It is therefore clear that the result of dropping $\I_1$ from the canonical window model of $W_2$ is exactly the canonical window model of $W_4$.\\
    To get the model for $W_3$ then, we add the canonical model $I_4$ of $\A_4$, which, again, is equal to its standard interpretation, and apply any rules possible. Clearly the second rule, $B\sqcap D \sqsubseteq E$ can be applied to $a\in B^{\I_4}$ and $a \in D^{\I_2}$, and therefore we add $a$ to $E^{\I_2}$. As no more rules can be applied, this results in the canonical window model for $W_3$.
\end{exmp}

We can now describe a semi-naive algorithm for adding a single ABox to a window materialization:
\begin{algorithm}
\caption{Adding one Abox to a window materialization}\label{alg:add_single_abox}
\begin{algorithmic}[1]
\Require A TBox $\T$, a window $W = \left\{\A_{t_1}, ..., \A_{t_{n-1}}\right\}$, its canonical model  $\left\{\I_1, ..., \I_{n-1}\right\}$ and a new ABox $\A_{t_n}$
\Ensure Canonical window model $M = \left\{\I_1, ..., \I_{n-1}, \I_{n}\right\}$ for $W'= W \cup \A_{t_n}$
\State $\Delta_1 \gets \I_n$\Comment{$\I_n$ is the canonical model of $\A_{t_n}$}
\State $M \gets \left\{\I_1, ..., \I_{n-1}\right\}$
\While{$\Delta_1 \neq \emptyset$}
    \ForAll{$B\sqsubseteq A$ in $\T$}
        \If{$\exists a \notin B^{\I_M}$ and $\exists S \subset \{1, ..., n\}$ such that $ a \in B^{\bigoplus\limits_{i\in S}\I_i}$ but $a\notin B^{\bigoplus\limits_{i\in S'}\I_i}$ for any $S'\subsetneq S$}
            \State $j \gets \min S$
            \State $A^{\I'_j} \gets A^{\I'_j}\cup \{a\}$ \Comment{$\I'_j$ is a new interpretation to which we add facts to be added to $\I_j$}
        \EndIf
    \EndFor
    \State $\Delta_2 \gets \{\I'_1, ..., \I'_{n-1}\}$
    \State $ M \gets \{\I_i\oplus \I'_i \vert 1\le i\le n, \I_i \in M, \I'_i \in \Delta_1\}$
    \State $\Delta_1 \gets \Delta_2$
\EndWhile
\end{algorithmic}
\end{algorithm}
Suppose $W = \left\{\A_{t_1}, ..., \A_{t_{n-1}}\right\}$ is a window and we possess its materialization, in the form of its canonical window model $\left\{\I_1, ..., \I_{n-1}\right\}$. Our goal is to compute the canonical window model of $W \cup \A_{t_n}$. According to Definition~\ref{def:CanWinModel}, we need to check for every rule whether there is an individual who satisfies the body according to some combination of interpretations in $\left\{\I_1, ..., \I_{n-1}\right\}\cup \I_n$. We know however, that the rules for which such a combination does not include $\I_n$, have already had their heads added to appropriate models. We therefore make use of a semi-naive approach that goes as follows: we set $\Delta_1 := \I_n$ in line 1 and we set $M = \left\{\I_1, ..., \I_{n-1}\right\}$ in line 2. Then as long as our delta set is nonempty, we look for all rules for which the body is satisfied by a combination of interpretations including at least one from $\Delta_1$ and at least one from $\left\{\I_1, ..., \I_{n-1}\right\}$ inline 4 and 5. In line 10 we set $\Delta_2 =\{\I'_1, ..., \I'_{n-1}$ to be the interpretations that contain exactly the heads of these rules that need to be added to $\I_1, ..., \I_{n-1}$ respectively (recall that the head must be added to the oldest interpretation used in its derivation), which are calculated in line 6 and 7. Having computed $\Delta_2$, we can add $\Delta_1$ to the window model in line 11 to get $\left\{\I_1, ..., \I_{n}\right\}$, and then set $\Delta_1$ equal to $\Delta_2$ and repeat the process by finding all the rules whose bodies are entailed by a set of assertions including at least one fact from this new window model and one fact from the new $\Delta_1$. The process terminates when $\Delta_1 = \emptyset$ as is stated in line 3. It is clear this process finishes as the window consists of a finite amount of ABoxes, each with finite amounts of assertions.\\
\\
We can then immediately apply this algorithm to obtain an algorithm for the incremental maintenance of the materialization (canonical window model) of a window of an RL stream:\\
\begin{algorithm}
\caption{Incremental window slide}\label{alg:incr_maintenance}
\begin{algorithmic}[1]
\Require $W_I$, its canonical window model $\left\{\I_{t}\right\}_{t\in I}$, and $W_{I'}$
\Ensure $M = \left\{\I_{t}\right\}_{t\in I'}$
\State $M \gets \left\{\I_{t}\right\}_{t\in I\cap I'}$
\ForAll{$\A_i$ with $i \in I'\setminus I$}
    \State Calculate $\I_i$
    \State Add $\I_i$ to $M$ by Algorithm~\ref{alg:add_single_abox}
\EndFor
\end{algorithmic}
\end{algorithm}
Given an interval $I =[t_0,t_1]$, its corresponding window $W_I$ and its canonical window model $\left\{\I_{t}\right\}_{t\in I}$, we want to compute the canonical window model for $W_{I'}$ where $I' = [t_0+t_s, t_1+t_s]$ with $t_s < t_0-t_1$. To do this, we start by removing the materialized facts in $[t_0, t_0+t_s]$ in line 1, obtaining the canonical window model for $W_{[t_0+t_s, t_1]}$: $\left\{\I_{t}\right\}_{t\in [t_0+t_s, t_1]}$. Then line 3 constructs the canonical model for each ABox in $W_{]t_1, t_1+t_s]}$, say $\I_1, ...\I_k$, which we then simply add to the window model one by one by using the above algorithm in line 4. The result is the canonical window model for $W_{I'}$.

\section{Incremental windows with non-recursive inconsistencies}\label{inconsistency-semantics}

In this section we propose an adaptation to our incremental window maintenance algorithm for RL with the purpose of handling inconsistencies. However, because of the difficulty of dealing with inconsistencies and calculating repairs, we limit ourselves to what we call non-recursive inconsistencies.\\
Non-recursive inconsistencies are, informally, inconsistencies which can be deduced directly from the ABox, without the need for any reasoning or materialization at run time. Formally we define a non-recursive negative inclusion as follows:

\begin{definition}[non-recursive negative inclusion]
    A negative inclusion $B\subset \bot$ is non-recursive in an RL TBox $\mathcal{T}$ iff there is a finite set of concepts $B_1, ..., B_n$ such that for any ABox $\A$: $\mathcal{T},\A \vDash B \subset \bot$ iff $B_1^\I, ..., B_n^\I=\emptyset$, where $\I$ is the standard interpretation of $\A$, i.e. the interpretation containing only the assertions in $\A$.
\end{definition}
Intuitively, non recursive negative inclusions can be rewritten so that they can be seen directly from the ABox at first glance, without the need for any reasoning to happen first. This is important because it follows from this definition that if all negative inclusions in a TBox are non-recursive, then we can rewrite the TBox by replacing each negative inclusion $B\subset \bot$ with its corresponding set $B_1\subset\bot, ..., B_n \subset \bot$. This rewritten TBox, which we shall refer to as in inconsistency-normal form, then allows us to start the materialization process by looking for inconsistencies, resolving them, and only then start reasoning and adding facts to the materialization. This is important as it avoids materializing facts deduced from assertions that need to be removed to repair inconsistencies.\\ 
When looking to add an ABox $\A$ to a window $W$, the first step will therefore always be to check which inconsistencies will arise. We now introduce some notation in order to work with these inconsistencies:\\
\begin{definition}
    A set of assertions $S\subset \A_W\cup \A$ is a conflict set if it is minimally inconsistent, i.e. it is inconsistent and all strict subsets $S'\subsetneq S$ are consistent.
\end{definition}
Note that if $S$ is a conflict set then there must be some negative inclusion $B\subset \bot$ in the inconsistency normal TBox such that $S\vDash B$. 
\begin{exmp}
    Let $\T =\{\forall R.B \sqsubseteq B, A\cap B \sqsubseteq \bot\}$. The negative inclusion $A\sqcap B \sqsubseteq \bot $ is recursive in $\T$: because of the rule  $\forall R.B \sqsubseteq B$, it is never possible to rewrite the TBox in a way that allows us to check at first glance whether $A \sqcap B$ is satisfied, for all possible Aboxes. Consider the ABox $\A=\{ A(a), B(b), R(a,b)\}$, which is inconsistent w.r.t. $\T$: $a\in \forall R.B$ which means $B(a)$ which together with $A(a)$ implies inconsistency. One could choose to rewrite $\T$ by adding the negative inclusion $A\sqcap \forall R.B \sqsubseteq \bot$, but this does not fix the problem: let $\A' = \{ A(a), B(c), R(a,b), R(b,c)\}$. This ABox is again inconsistent w.r.t. to $\T$, as $B(c)$ together with $R(b,c)$ implies $B(b)$, which then together with $R(a,b)$ implies $B(a)$ which implies inconsistency as before. Moreover the added rule $A\sqcap \forall R.B \sqsubseteq \bot$ does not allow us to deduce inconsistency of $\A'$ at first glance. For this, we would need a rule $A\sqcap \forall R.\forall R.B \sqsubseteq \bot$. Naturally, we can continue this process infinitely, which means an infinite number of rules would need to be added, and we conclude that the negative inclusion $A\sqcap B \sqsubseteq \bot $ is indeed recursive.
\end{exmp}
For a conflict set $S$, we denote by $\min S$ the set of oldest assertions in $S$. This notation is useful as, generally, we will try to repair a window that contains $S$ by deleting element(s) of $\min S$. The following lemma builds on this idea, by helping us determine for which conflict sets it is necessary to delete $\min S$:
\begin{lemma}
    Let $S_1, ..., S_n$ be the conflict sets that arise when adding $\A$ to $W$. Then for any $1\le i \le n$, if there is no $j\ne i$ such that $\min S_j \subset S_i$, we have the following: for any $a\in \min S_i$ there is a window repair of $W\cup\{\A\}$ that does not contain $a$.
\end{lemma}
\begin{proof}
    Let $S_i$ be such that for all $S_j$, $\min S_j \not\subset S_i$. Then take a element of each $\min S_j \setminus S_i$, and an element $a \in \min S_i$, and remove them from $\A\cup \A_W$. We obtain a consistent subset $B_0$ of $\A\cup \A_W$. There must therefore be a repair $R_0$ such that $B_0\subseteq_P R_0$. Either $B_0 =R_0$ and we have the desired repair, or there is some $t_0$ such that $B_0\cap \A_{t_0} \subsetneq R_0\cap \A_{t_0}$ and $B_0\cap \A_{t} = R_0\cap \A_{t}$ for all $t>t_0$. We then set $B_1 = B_0 \cup (R_0\cap \A_{t_0})$. Note that $B_1$ is consistent: suppose there was a conflict set $S_j \subset B_1$. Then the element of $\min S_j$ we deleted must be contained in $R_0\cap \A_{t_0}$, and therefore $S_j \subset \bigcup\limits_{t\ge t_0}\A_t$. But since $B_1\cap \A_t = R_0\cap \A_t$ for $t\ge t_0$, this would imply $S_j \subset R_0$ which is not possible.\\
    Now again there must be some repair $R_1$ such that $B_1\subseteq_P R_1$. Either $B_1 = R_1$ or there is some $t_1$ such that $B_1\cap \A_{t_1} \subsetneq R_1\cap \A_{t_1}$ and $B_1\cap \A_{t} = R_1\cap \A_{t}$ for all $t>t_1$. Note that $t_1 < t_0$ as otherwise $R_0\subseteq_PR_1$ and $R_1\ne R_0$ which contradicts the maximality of $R_0$. Again, we set $B_2 = B_1 \cup (R_1\cap \A_{t_1})$. Note now that $B_2$ is consistent: suppose there was a conflict set $S_j \subset B_2$. Then the element of $\min S_j$ we deleted must be contained in $R_1\cap \A_{t_1}$ as we have not added any elements older than $t_1$ to $B_0$. Therefore $S_j \subset \bigcup\limits_{t\ge t_1}\A_t$. But since $B_2\cap \A_t = R_1\cap \A_t$ for $t\ge t_1$, this would imply $S_j \subset R_1$ which is not possible.\\
    We can repeat this process, but it is clear this must end at some point (we deleted $n$ elements to get $B_0$ and therefore we can add at most $n$ back, in reality even less as adding all $n$ back would result in inconsistency). Therefore, we obtain some $B_k$ that is a repair, which we constructed only by adding elements back to $B_0$. Moreover, $a\notin B_k$, as $S_i\setminus \{a\} \subset B_0\subset B_k$. Therefore, $B_k$ is the desired repair.\qed
\end{proof}
Using this lemma, we can adapt our procedure for adding an ABox to a window materialization, such that it can resolve non-recursive inconsistencies. 
Suppose again that $W = \left\{\A_{t_1}, ..., \A_{t_{n-1}}\right\}$ is a window and we posses its materialization, in the form of its canonical window model $\left\{\I_1, ..., \I_{n-1}\right\}$. Suppose also that $\A_W$ is consistent and that $\mathcal{T}$ does not contain recursive negative inclusions, and is in inconsistency-normal form. Before adding any facts to the materialization, we first check if the bodies of any of the negative inclusions in $\mathcal{T}$ are entailed by $\A_W\cup \A$. If there are no conflicts, the procedure can proceed as in algorithm~\ref{alg:add_single_abox}.\\
Suppose now there are conflict sets $S_1, ... S_k$. We will construct the intersection of all window repairs of $\A_W \cup \A$ as follows: we start by arranging the conflict sets from newest to oldest, i.e. $\min S_i$ is more recent than (or equally recent as) $\min S_{i+1}$. We will resolve conflicts starting with the most recent and working towards the oldest ones. This is because of the condition in the above lemma: \[\min S_j \subset S_i\] can only hold if $S_j$ is more recent than (or equally recent as) $S_i$. Therefore, the most recent conflict sets (those whose $\min$ are part of the ABox with the largest timestamp), just need to be checked against each other while the oldest conflict sets would need to be checked against all others. We illustrate this idea with a short example:\\
\begin{exmp}
    Consider a window that consists of ABoxes $\A_1 =\{A(a)\}, \A_2= \{B(a)\}$ and an ABox $\A_3 = \{C(a), D(a)\}$ to be added to the window. Let $\T = \{ A\sqcap B\sqcap C \sqsubseteq \bot, B\sqcap D\sqsubseteq \bot\}$. Then adding $\A_3$ to the windows results in 2 conflict sets: $S_1 = \{A(a), B(a), C(a)\}$ and $S_2 =\{B(a), D(a)\}$. In order to resolve these conflicts, we must arrange them in order of recency i.e., we start with the conflict set with the most recent $\min$, which in this case is $S_2$. We see that $\min S_2 = \{B(a)\}$, and that there are no conflict sets equally recent as $S_2$, therefore we know that $B(a)$ must be excluded from all window repairs, and in particular, from the intersection of all window repairs. Now we move on to consider $S_1 = \{A(a), B(a), C(a)\}$, but $\min S_2 \subset S_1$, and therefore $S_1$ is resolved in every repair by resolving $S_2$. This means no action needs to be taken on $S_1$, and we conclude that the desired intersection of all repairs is obtained by removing $B(a)$, which is coincidentally the only repair.
\end{exmp}
Now let $S_1, ..., S_{p_1}$ be the most recent conflict sets. It is clear that if $\vert\min S_i\vert = 1$ holds for $1\le i \le p_i$, then every repair must remove this element. Let $D$ be the set of all such elements: \[D = \bigcup\limits_{1\le i \le p-1}^{\vert\min S_i\vert = 1}\min S_i\] The intersection of $D$ with any repair must be empty. Therefore, for any $1\le i\le n$, if $S_i \cap D \ne \emptyset$, then $S_i$ is resolved already by deleting $D$. This means we can drop such $S_i$ from consideration for our repairing procedure. Then consider all $S_i$ with $1\le i \le p_1$ such that $S_i\cap D = \emptyset$. For these, we have $\min \vert S_i \vert >1$, which makes them tricky to deal with, as different repairs might keep or delete different elements. However, our previous lemma lets us deal with these very easily: if there is no $S_j$ such that $\min S_j \subset S_i$, then the intersection of all repairs must be disjoint with $\min S_i$. If, on the other hand there is such an $S_j$, then we can drop $S_i$ from consideration: Indeed if such an $S_j$ exists, then any repair must remove some element in $\min S_j$, as it must be one of the recent repairs, and this will always repair $S_i$. Consequently, the only $S_i$ with $1\le i \le p_1$ such that $S_i\cap D = \emptyset$ that we need to consider are those for which $\min S_i$ is minimal with respect to the $\subseteq$ preorder. We gather all these in the following set: \[N = \left\{\min S_i \vert  1\le i \le p_1, S_i\cap D = \emptyset, \min S_i \text{ is minimal w.r.t }\subseteq  \right\}\]
Our lemma says that the intersection of all repairs is disjoint with any of the sets in $N$. Moreover, we know that no repair fully contains any of the sets in $N$. Therefore, for any conflict set $S_i$ with $1\le i \le n$, if there is a $M \in N$ such that $M \subset S_i$, we can drop $S_i$ from consideration.\\
\begin{algorithm}
\caption{Resolving inconsistencies when adding an ABox}\label{alg:single_abox_incon_repair}
\begin{algorithmic}[1]
\Require $\T$ in recursion normal form, $W= \left\{\A_{t_1}, ..., \A_{t_{n-1}}\right\}$, $\A$ to be added
\Ensure $D$ s.t. $(\A_W\cup \A) \setminus D$ is a window repair
\State $C \gets [\hspace{2mm}]$
\ForAll{$B\sqsubseteq \bot \in \T$}
    \If{$\A_w\cup\A\vDash B$}
        \State Add to $C$ all $S_i$ for which $S_i\vDash B$, $\forall S'_i\subsetneq S_i, S'_i\not\vDash B$
    \EndIf
\EndFor
\State $D \gets [\hspace{2mm}]$
\State $N \gets [\hspace{2mm}]$
\While{$C \neq \emptyset$}
\State $C_1 \gets$ the most recent $S_i$'s in $C$ \Comment{i.e., those whose $\min S_i$ has the highest time stamp $t$}
\State $D\gets D \cup \bigcup\left\{\min S_i \vert S_i\in C_1, \vert \min S_i\vert = 1\right\}$
\State $C\gets C \setminus \left\{S_i \vert S_i\in C, S_i\cap D \neq \emptyset \right\}$
\State $N\gets N \cup \left\{\min S_i \vert S_i\in C_1, \forall S_j\in C_1\min S_j \not\subset \min S_i \right\}$
\State $C \gets C \setminus \left\{S_i \vert S_i\in C, \exists M\in N, M\subseteq S_i\right\}$
\EndWhile
\State $D\gets D\cup \bigcup\limits_{M\in N}M$
\end{algorithmic}
\end{algorithm}

We now iteratively reapply this idea to resolve all conflicts: of the conflict sets that remain in consideration, take the most recent ones. Of those most recent ones, take those with a single oldest assertion. Add these assertions to $D$, and remove from consideration any conflict sets that have a non empty intersection with $D$. Take the most recent conflict sets with multiple oldest assertions that remain in consideration. Of those, take the ones whose $\min$ is minimal w.r.t $\subseteq$, and remove the others from consideration. Add all the $\min$ to $N$, and finally remove all conflict sets that fully contain a set in $N$ from consideration. \\
It is clear that applying this will result eventually in removing all conflicts from consideration, leaving us with a set $D$ and a set $N$. To obtain the intersection of all repairs, all that remains to do is to remove $D \cup \bigcup\limits_{M\in N} M$ from $\A_W\cup\A$. Lastly, we need to remove facts that are deduced from this set, from the materialization. This can in practice be done by using any of the algorithms discussed in \cite{motik2019maintenance}. 

\begin{lemma}
    Given an interval $I= [t_0, t_1]$, and a $t_i<t_1\in I$, let $W_I$ be the window over a stream $S$ defined by $I$, and let $W'_I$ be a window repair of $W_I$. Let $I' = [t_i,t_1]$, and let $W_{I'}$ be the window over $S$ defined by $I'$. The set $W'_{I'} = \left\{\A'_t\right\}_{t\in I'}\subset W'_I$ is a window repair of $W_{I'}$.
\end{lemma}
\begin{proof}
    Suppose $W'_{I'} $ is not a window repair of $W_{I'}$. Then, because a subset of a consistent ABox is again consistent, there must be a smallest $t\in I'$ such that $\bigcup\limits_{t_i\le s\le t}\A'_s$ is not a maximal consistent subset of $\A_t\cup\bigcup\limits_{t_i\le s< t}\A'_s$. There must therefore be an assertion $a\in\A_t\setminus\A'_t$ such that $\{a\}\cup\bigcup\limits_{t_i\le s\le t}\A'_s$ is consistent. However, this would imply that $W'_I$ is not a window repair of $W_I$, as $\bigcup\limits_{t_0\le s\le t}\A'_s\subseteq_W \{a\}\cup\bigcup\limits_{t_i\le s\le t}\A'_s$. This contradiction concludes the proof.\qed
\end{proof}
This lemma shows that when a window slides and we need to drop the oldest ABoxes, we can drop the last ABoxes from our repair and still have a window repair as starting point for the repair of the next window.\\
\\
\section{Discussion}
The incremental foundations of our new semantics ensure that the algorithms are capable of handling the challenge that is the high-velocity nature of streams. Moreover, these ideas were developed with edge-computing in mind, where materialization of momentary ABoxes would be done at the edge before being sent to the cloud. If not done on the edge, the momentary ABoxes could be materialized in parallel on the central system before adding them to the window. Either way, these semantics are fundamentally suited for high performance reasoning on streams.\\
The issue of inconsistencies is resolved by relying on the core principle "new is always better" which provides a clear-cut and sensible way of dealing with inconsistent data. This principle allows us to perform repairs in a more deterministic manner, rather than just guessing at which assertions need to be removed. Performing repairs incrementally also ensures that the repairing process is low-latency and in line with the idea of "disproving a disproof is not a proof".\\
The semantics and algorithms we presented solve the challenges of working incrementally and handling inconsistencies, with some limitations. The biggest limitation is that only non-recursive inconsistencies are handled. However, in practice, this can be remedied by including rewritten version of recursive negative inclusions up to some predetermined depth. This would allow the algorithm for inconsistency repair to handle most recursive inconsistencies, depending on the specified maximum depth, at the cost of more rules needing checking which would likely impact performance. Another limitation is that these semantics consider only streams, however, in practice, streams are often reasoned over in combination with some static "background" knowledge base. The proposed semantics could be modified to this scenario in a fairly minimal way, by having a background ABox that is the starting ABox of every window and is of the highest priority in the repair order, i.e. it has timestamp $+\infty$. This would ensure that no background knowledge is ever removed, as background knowledge should be verified to be true at all times during the stream.
\section{Conclusion}

In conclusion, we have proposed semantics for reasoning incrementally over DL streams, while handling (non-recursive) inconsistencies. We have proposed semi-naive algorithms for reasoning under these semantics for the RL description logic underlying OWL2 RL. While the proposed semantics lacks the immediate capabilities of handling background knowledge and recursive inconsistencies, they can be adapted to both in theory, which could be of interest in practical use-cases. Furthermore, experimental evaluation of these algorithms remains to be performed and as such is a core goal of our future work. Lastly,  the adaptation of these semantics to other description logics, in particular the DLs underlying the other two OWL2 profiles is a major challenge. Both of these DLs lack a general finite "minimal' model like the canonical model of RL knowledge bases, however they do have similar properties that we believe should allow for adaptation of our semantics.

%
%
%
\bibliographystyle{splncs04}
\bibliography{biblio}
\end{document}